\documentclass[mlmain]{jmlr}
\usepackage{wrapfig}

\usepackage{longtable}

\usepackage{booktabs}
\usepackage[load-configurations=version-1]{siunitx} 

\theorembodyfont{\upshape}
\theoremheaderfont{\scshape}
\theorempostheader{:}
\theoremsep{\newline}

\jmlrvolume{}
\firstpageno{1}
\editors{Francisco Acosta, Simone Azeglio, Chase van de Geijn, Christian Shewmake, Sophia Sanborn, and Nina Miolane.}

\jmlryear{2025}
\jmlrworkshop{Symmetry and Geometry in Neural Representations}

\title[Synergistic Computational Graph Effects in Multi-Head Attention]{Beyond Parallelism: Synergistic Computational Graph Effects in Multi-Head Attention}

  \author{\Name{Haitz Sáez~de~Ocáriz~Borde} \Email{ocariz@supermodel.ai} \\
  \addr Supermodel
}

\begin{document}

\maketitle

\begin{abstract}
Multi-head attention powers Transformer networks, the primary deep learning architecture behind the success of large language models~(LLMs). Yet, the theoretical advantages of multi-head versus single-head attention, beyond mere parallel processing, remain underexplored. In this paper, we reframe multi-head attention as a system of potentially synergistic computational graphs, where each head functions as a feedforward directed acyclic graph (DAG) with a common sink state. We provide intuition and preliminary theoretical analysis of mixing time and minimax fidelity in this framework. Our results show that multi-head attention can synergistically enhance information propagation, yielding faster mixing times and minimax fidelity amplification under specific head-diversity conditions. Finally, we train single-head and multi-head Transformers, each with the same total number of parameters, on sequence manipulation tasks and empirically verify the predicted effects. The code is available at \url{https://github.com/haitzsaezdeocariz/beyondparallelism}.
\end{abstract}

\vspace{-5pt}
\section{Introduction}

In this paper, we adopt a graph‐theoretic lens to interpret how multi‐head attention processes and propagates information. By modeling each head as a feedforward computational graph, we expose the synergistic computational pathways enabled by parallel attention and introduce concrete metrics (mixing time and minimax fidelity). We anticipate that these insights may inspire future actionable interpretability techniques, empowering practitioners to diagnose, explain, and optimize attention-based architectures, while also helping us understand how specific behaviors and computations arise.

\begin{wrapfigure}[11]{r}{0.3\linewidth}
    \centering
    \includegraphics[width=\linewidth]{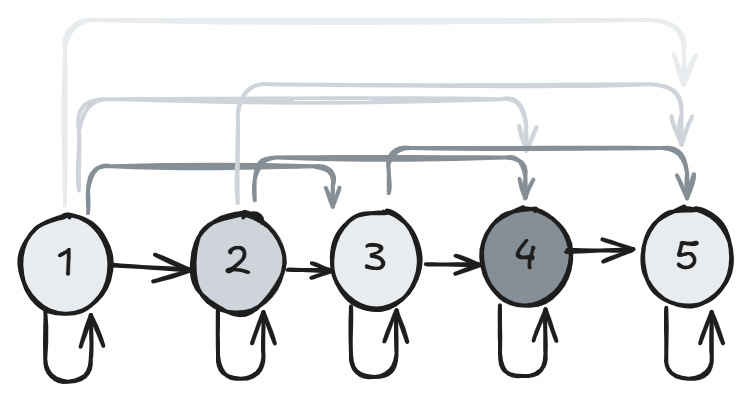}
    \caption{Example feedforward DAG with $n=5$ nodes.}
    \label{fig:ff}
\end{wrapfigure}

In a standard (decoder-only) Transformer with causal (masked) attention~\citep{Vaswani2017AttentionIA,Brown2020LanguageMA}, which is omnipresent in language modeling among other applications, each position in the sequence can only attend to itself and to earlier positions. Such dependency can be abstracted via a structure with edges going strictly forward in time: one can view each head's pattern of attention as a feedforward directed acyclic graph (DAG) on \(n\) nodes (with the first node having no incoming edges other than its self-loop). The position \(n\) is the DAG's sink, since it has edges coming in from all previous positions but does not feed forward to any later positions, see Figure~\ref{fig:ff}.

In~\cite{what}, the authors use a DAG to model the entire forward pass of the network. On the other hand, in this paper, we will use a similar argument to compare the (intra-layer) computational graph of \textit{single-head} attention against that of the ubiquitous \textit{multi-head} attention, where multiple distinct dynamic graphs (the graph connectivity strength is input dependent) are used for simultaneous information propagation. We theoretically study \textit{mixing time} and \textit{minimax fidelity}. Additionally, we complement our findings by training single-head and multi-head Transformers on toy datasets and computing empirical proxies for these quantities. We indeed find that multi-head attention exhibits improved performance as compared to single-head using the same number of trainable parameters.

\vspace{-10pt}
\section{Multi-Head Mixing Time}

Briefly, the mixing time is defined as the time it takes for a probability distribution over states to converge to a stationary distribution. In \cite{what}, the authors argue that the mixing time is small if and only if there are numerous paths from most nodes in the graph to the sink. They also claim that since having many paths facilitates efficient data propagation, the mixing time serves as a meaningful measure of how efficiently our network will operate under this computational graph. Hence, a lower mixing time would be better since it generally implies statistical efficiency. In this section we explore the following question: \textit{Does mixing time improve (go down) when we combine multiple parallel feedforward computational graphs?}
\vspace{-10pt}
\subsection{Multi-Head Stationary Distributions}
\label{subsec:Multi-head Stationary Distributions}

\begin{definition}[Feedforward graph] A \emph{feedforward graph} on $n$ vertices is a DAG in which 
the vertices can be indexed $1,2,\dots,n$ so that all edges $(j \to i)$ satisfy $j \le i$.
\end{definition}

\begin{definition}[Unique sink] A \emph{unique sink} $\tau$ (at position $n$) is a vertex with no outgoing edges to \emph{distinct} vertices 
(i.e.\ the only possible edge is the self-loop $(\tau \to \tau)$). 
\end{definition}

\begin{definition}[Random walk matrix] Given a directed graph $G = (V,E)$ on $n$ vertices, the \emph{random walk matrix} $W$ 
is the $n \times n$ matrix with entries
\begin{equation}
  W_{ij} \;=\;
  \begin{cases}
    \frac{1}{\delta_j^{\to}}, & \text{if }(j \to i) \in E,\\
    0, & \text{otherwise},
  \end{cases}
\end{equation}
where $\delta_j^{\to}$ is the outdegree of node $j$. For causal attention this matrix is lower-diagonal. 
\end{definition}

\begin{definition}[Stationary distribution]
A probability distribution $\pi \in \mathbb{R}^n$ on the vertices is called a \emph{stationary distribution} 
if $W \,\pi = \pi$. 
\end{definition}

\begin{lemma}[Stationary Distribution for a Single-Head Unique Sink]
  \label{lem:single-head}
  Let $G$ be a feedforward graph on $n$ vertices with a unique sink $\tau$. 
  Then the only stationary distribution for the random walk matrix $W$ is $1_{\tau}$, 
  the distribution taking value $1$ at $\tau$ and $0$ elsewhere~\citep{what}.
  \label{single-head}
\end{lemma}

As discussed in~\cite{what}, clearly $W\,1_{\tau} = 1_{\tau}$ because $\tau$ has no outgoing edges to other nodes (only possibly a self-loop). In terms of uniqueness, suppose $\pi$ is another stationary distribution with $\pi \neq 1_{\tau}$. Let $j < n$ be the smallest-index vertex with $\pi_j \neq 0$. Because $G$ is feedforward, $j$ has at least one outgoing edge (apart from its self-loop) $(j \to i)$ with $i > j$ unless $j=\tau$. This implies $\pi_j$ will ``leak forward'' under $W$, contradicting stationarity. Thus no such $\pi$ can exist.

Next, we extend the unique-sink argument in~\cite{what}. Specifically, we adapt the original derivation used for the full computational graph to model a single attention head in our analysis (as in Lemma~\ref{single-head}) and then compare it to the multi-head attention setting. 

\begin{lemma}[Stationary Distribution for a Multi-Head Unique Sink]
  \label{lem:multi-head}
  \\ Let $G^{(1)},\dots,G^{(H)}$ be $H$ feedforward graphs on $n$ vertices, 
  each having the \emph{same} unique sink $\tau$. Suppose the final operation merges the heads (via concatenation and a linear projection or any acyclic merging). 
  Then the only stationary distribution (over the combined state space) 
  is the one that places all its mass at $\tau$.
\end{lemma}

\begin{proof}
    See Appendix~\ref{app:Proofs}.
\end{proof}

Thus, even with multiple heads, as long as each head is a feedforward DAG sharing the 
same sink $\tau$ (see Figure~\ref{fig:multiheadsink}), the unique stationary distribution argument holds.

\begin{wrapfigure}[15]{r}{0.45\linewidth}
\centering    \includegraphics[width=\linewidth]{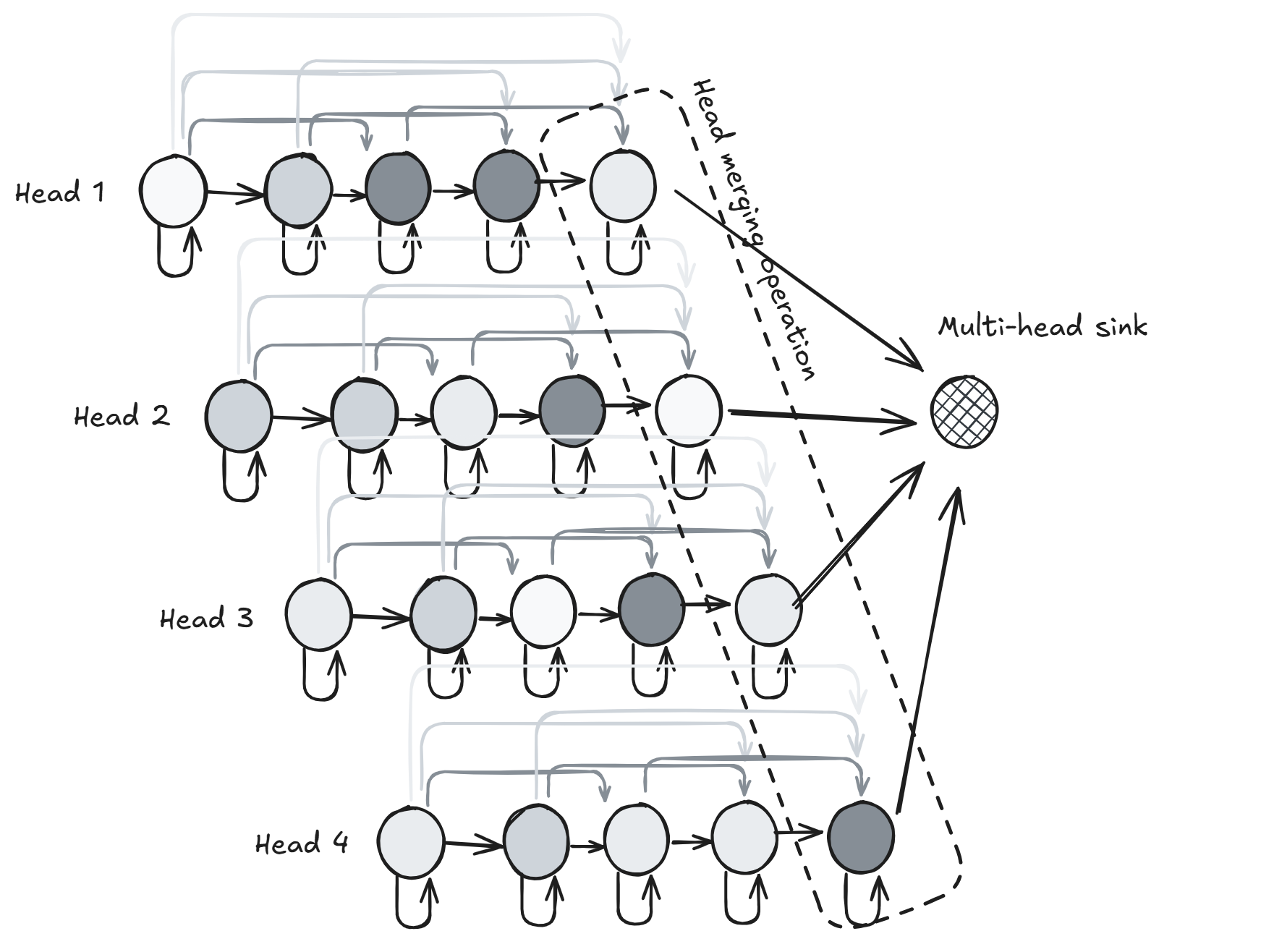}
    \caption{Multi-head sink visualization.}
    \label{fig:multiheadsink}
\end{wrapfigure}

\subsection{Having Multiple Heads can Improve Mixing Time}

\begin{definition}[Mixing Time]
\label{def:mixing_time}
Given a random walk matrix $W_{ij}$ with stationary distribution $\pi$, the \emph{mixing time} is defined as:
\begin{equation}
    T_{\mathrm{mix}}(\epsilon) = \min\{t \geq 0 : \max_{j} \| W^t_{ij} - \pi \|_{\mathrm{TV}} \leq \epsilon \},
\end{equation}
where $\|\cdot\|_{\mathrm{TV}}$ denotes the total variation distance, and $\epsilon > 0$ is a convergence threshold.
\end{definition}

The goal is to identify the earliest possible time where the distribution is sufficiently close (within $\epsilon$) to the stationary distribution $\pi$. Hence, taking the minimum over $t$ ensures we are finding the fastest convergence time. Additionally, the maximum over all initial states $j$ ensures a worst-case scenario analysis. In other words, mixing time guarantees convergence from every possible starting state: it reflects the slowest possible convergence to the stationary distribution across all starting points.

To address the main question of this section, we derive an upper bound on the mixing time for a combined multi-head attention system using combinatorial and probabilistic arguments rather than relying on the standard spectral gap analysis such as those in~\cite{Horn1985MatrixA,Levin2017MarkovCA}. This is because as pointed out by~\cite{what}, the latter does not necessarily apply given that our random walk matrices are lower-diagonal.

Let \(G^{(1)}, \dots, G^{(H)}\) be \(H\) feedforward graphs on the same set of \(n\) vertices, each with a unique sink \(\tau\). Suppose that, for each head \(h\), the probability of making a forward move (i.e., a move that brings the state closer to \(\tau\)) is at least \(p_h\). Assume that the heads are combined via a convex combination $\overline{W} = \sum_{h=1}^{H} \alpha_h\,W^{(h)},$ with \(\alpha_h\ge 0\) and \(\sum_{h=1}^H \alpha_h = 1.\) Define the effective forward probability as $p = \sum_{h=1}^{H} \alpha_h\,p_h.$

\begin{theorem}[Multi-Head Mixing Time Bound via Forward Moves]
\label{thm:multi-head-mixing} If it requires at most \(N = n-1\) forward moves to reach the sink \(\tau\) (with the worst-case being the leftmost node), then with high probability the mixing time of the combined chain satisfies
$T_{\mathrm{mix}}(\overline{W},\epsilon) \lesssim \frac{2N}{p}$, where \(\epsilon\) is a small constant. In particular, note that since \(p \le \max_{1\le h\le H} p_h\) it follows that $\frac{2N}{p} \ge \frac{2N}{\max_{1\le h\le H} p_h}$. Thus, if the convex weights are chosen arbitrarily, the bound on the mixing time of the combined chain is in general no better than the mixing time of the fastest individual head, but always better than the worse. However, if the weights can be selected adaptively:
\begin{equation}
\begin{split}
T_{\mathrm{mix}}(\overline{W},\epsilon) \lesssim \frac{2N}{\max_{1\le h\le H} p_h} = \min_{1\le h\le H} T_{\mathrm{mix}}\bigl(W^{(h)}, \epsilon\bigr).
\end{split}
\end{equation}
\end{theorem}

\begin{proof}
    See Appendix~\ref{app:Proofs}, including scope (intra-layer analysis) and modeling assumptions (convex proxy).
\end{proof}

In summary, each head \(W^{(h)}\) corresponds to a Markov chain that propagates information along selected edges in the DAG. By merging these heads, additional transition edges are added, reducing potential bottlenecks. While combining heads via a convex combination may not always yield a faster mixing time than the best individual head (it is only guaranteed to be faster than the worse head), adaptive weighting (favoring the head with the highest \(p_h\)) can ensure that the bound on the mixing time of the combined chain is close to the fastest head. We empirically verify this effect in Section~\ref{sec:Experimental Validation}.

\vspace{-10pt}
\section{Multi-Head Minimax Fidelity}
\label{sec:Multi-Head Minimax Fidelity}

In a similar spirit to our analysis of multi-head mixing time, we now turn to studying the fidelity of information propagation when multiple feedforward attention heads are combined. As pointed out in~\cite{what}, whereas mixing time measures how quickly probability mass converges to the unique sink, fidelity quantifies how sharply the signal from each node is preserved when it reaches the sink.

\begin{definition}[Diffusion Matrix] For a single-head feedforward graph, we define a \emph{diffusion} matrix
\begin{equation}
  \Delta_{ij} \;=\; 
  \begin{cases}
    \frac{1}{\delta_i^{\leftarrow}}, & \text{if } (j\to i) \in E,\\[1mm]
    0, & \text{otherwise},
  \end{cases}
\end{equation}
which normalizes by the column (i.e., the in-degree $\delta_i^{\leftarrow}$, which is always positive since we always have self-loops).
\end{definition}

Hence, in this notation index $i$ corresponds to the receiving node and the second index $j$ corresponds to the sending node. In our analysis, the receiving node we are interested in is the sink, $i=\tau$. 

\begin{definition}[Signal at the Sink]
If we initialize with a one-hot vector \( e_j \) (placing all mass at node \( j \)), then after \( t \) steps the signal at the sink \(\tau\) is given by $\phi^{(h)}_j(t) \;=\; \bigl((\Delta^{(h)})^t\bigr)_{\tau j}.$
\end{definition}


\begin{definition}[Node Fidelity] The \emph{node fidelity} is defined as the maximum signal that node \( j \) ever contributes to \(\tau\): $\phi^{(h)}_j \;=\; \max_{t} \phi^{(h)}_j(t) =\; \max_{t} \bigl((\Delta^{(h)})^t\bigr)_{\tau j}.$
\end{definition}

\begin{definition}[Optimal Fidelity Time] We define the \textit{optimal fidelity time} as the time at which node $j$ (head $h$) attains its maximum signal: $t_{j}^{(h)} = \arg\max_{t}\left\{ \phi_j^{(h)}(t) \right\}.$
\end{definition}

Note that the amount of signal that reaches the sink $\tau$ does not necessarily increase monotonically, see Proposition 5.1 in~\cite{what}. Instead, it can rise to a peak at some intermediate time and then decay as the signal becomes diluted or averaged out by further propagation. Therefore, even if the signal eventually decays (due to the averaging nature of the diffusion process), its peak value, captured by the equation above, represents the most effective transmission from node $j$. Lastly, by taking the minimum over all nodes we identify the worst-case (least preserved) signal among all nodes: 

\begin{definition}[Minimax Fidelity]
The \emph{minimax fidelity} for a single head \( h \) at the sink $\tau$ is given by $\phi^{(h)}_{\min} \;=\; \min_{j} \phi^{(h)}_j.$
\end{definition}

Intuitively, this would be the lowest peak signal produced at the sink by any node in the DAG (token in the sequence).

\subsection{Having Multiple Heads can Amplify Fidelity}

Next, we will show that it is possible to observe synergistic effects in terms of signal propagation over parallel heads.

\begin{definition}[Multi-Head Diffusion Operator] Let $\Delta^{(1)},\dots,\Delta^{(H)}$ be the diffusion matrices corresponding to a total of $H=|\mathcal{H}|$ feedforward attention heads. We define \textit{multi-head diffusion operator} using convex weights $\{\beta_h\}_{h\in \mathcal{H}}$ (i.e., $\beta_h\ge0$ and $\sum_{h\in \mathcal{H}}\beta_h=1$) as
\begin{equation}
\overline{\Delta}_{ij} = \sum_{h\in \mathcal{H}}\beta_h\,\Delta^{(h)}_{ij}.
\end{equation}
\end{definition}

\begin{remark}[Cross‐Head Pathway Contribution] Taking powers of the multi-head diffusion operator yields the following expression (note non-commutativity: in general $\Delta^{(h)}\Delta^{(h')} \neq \Delta^{(h')}\Delta^{(h)}$) which promotes additional information propagation pathways:
\vspace{-5pt}
\begin{equation}
\begin{aligned}
(\overline\Delta)^t
&= \Bigl(\sum_{h\in\mathcal{H}}\beta_h\Delta^{(h)}\Bigr)^t = \sum_{h_1\in\mathcal H} \cdots\sum_{h_t\in\mathcal H}
   \Bigl(\prod_{r=1}^{r=t}\beta_{h_{r}}\Bigr)\,
   \Delta^{(h_1)}\cdots\Delta^{(h_t)}.
\end{aligned}
\end{equation}
\end{remark} 

\begin{example}
Let our nodes be ordered $(u,v,\tau)$. Head 1 has edges $u\to v$ (plus self‐loops everywhere), head 2 has edges $v\to\tau$ (plus self‐loops). Then the diffusion matrices are:
$$
\Delta^{(1)} = 
\begin{pmatrix}
1   & 0   & 0 \\[3pt]
\tfrac12 & \tfrac12 & 0 \\[3pt]
0   & 0   & 1 
\end{pmatrix}, 
\qquad
\Delta^{(2)} = 
\begin{pmatrix}
1   & 0     & 0 \\[3pt]
0   & 1     & 0 \\[3pt]
0   & \tfrac12 & \tfrac12
\end{pmatrix}.
$$

For a single head: $
   \bigl(\Delta^{(1)}\bigr)^2_{\tau u}
   = 0,
   \qquad
   \bigl(\Delta^{(2)}\bigr)^2_{\tau u}
   = 0,
   $ because head 1 never has a path $u\to\tau$ in two steps, nor does head 2. On the other hand when using multi-head attention we obtain cross‐head product terms:

   $$
   \bigl(\Delta^{(2)}\Delta^{(1)}\bigr)_{\tau u}
   = \tfrac12 \times \tfrac12
   = \tfrac14
   > 0.
   $$

The term corresponding to head 1 then head 2 appears with weight $\beta_1\beta_2$:

\begin{equation*}
\begin{aligned}
(\overline\Delta)^2_{\tau u} =
   \\ = & \, (\beta_1^2\,(\Delta^{(1)})^2)_{\tau u} + ((\beta_1\,\Delta^{(1)})(\beta_2\,\Delta^{(2)}))_{\tau u} +
     \\ + &  \,((\beta_2\,\Delta^{(2)})(\beta_1\,\Delta^{(1)}))_{\tau u}
     + (\beta_2^2\,(\Delta^{(2)})^2)_{\tau u} =
     \\ + & 0 + 0 + \beta_1\beta_2\frac{1}{4} + 0 = \beta_1\beta_2\frac{1}{4} > 0. \,
\end{aligned}
\end{equation*}
   
so the multi-head diffusion operator captures a contribution. In other words, whenever one head bridges $u\to v$ and another bridges $v\to\tau$, their composition opens up a two‐step path that neither head has alone.
\end{example}

This would exemplify the trivial case where the fidelity is zero for each individual head, but not for the multi-head diffusion operator.

\begin{definition}[Multi-Head Node Fidelity]
We define the  \textit{multi-head node fidelity} as the node fidelity computed based on the multi-head diffusion operator: $\phi^{\mathrm{multi}}_j = \max_{t}\Bigl((\overline{\Delta})^t\Bigr)_{\tau j}.$
\end{definition}

\begin{definition}[Multi-Head Minimax Fidelity]
The \emph{multi-head minimax fidelity} is given by $\phi^{\mathrm{multi}}_{\min} \;=\; \min_{j} \phi^{\mathrm{multi}}_j.$
\end{definition}

\begin{definition}[Best Head Minimax Fidelity] Comparing the minimax fidelity across heads in the multi-head attention mechanism, the best head minimax fidelity is: $\phi^{(h_*)}_{\min} = \max_{h\in\mathcal{H}} \phi^{(h)}_{\min}.$
\end{definition}

\begin{definition}[Best Head] Consequently we define the \textit{best head} as: $h_* = \arg\max_{h}\left\{ \phi_{\min}^{(h)} \right\}.$
\end{definition}

Next we proceed to show a non-trivial case in which the best head minimax fidelity is not zero (there exists a path connecting all nodes to the sink), yet the multi-head minimax fidelity is higher:

\begin{example}
Let our nodes be ordered $(u,v,w,\tau)$. Let head 1 be a linear chain with edges $u\to v, v \to w, w \to \tau$ (plus self‐loops everywhere), and head 2 a feedforward DAG with $u\to w, v\to w, v \to \tau, w \to \tau$ (plus self‐loops). Also set $\beta_1=\beta_2 = \frac{1}{2}.$ Then the diffusion matrices and the multi-head diffusion operator are:
$$
\Delta^{(1)} = 
\begin{pmatrix}
1   & 0   & 0 & 0\\[3pt]
\tfrac12 & \tfrac12 & 0 & 0 \\[3pt]
0   & \tfrac12   & \tfrac12 & 0 \\[3pt]
0   & 0   & \tfrac12 & \tfrac12 \\[3pt]
\end{pmatrix}, 
\qquad
\Delta^{(2)} = 
\begin{pmatrix}
1 & 0 & 0 & 0\\[3pt]
0 & 1 & 0 & 0\\[3pt]
\tfrac13 & \tfrac13 & \tfrac13 & 0\\[3pt]
0 & \tfrac13 & \tfrac13 & \tfrac13
\end{pmatrix},
\qquad
\overline{\Delta} = 
\begin{pmatrix}
1 & 0 & 0 & 0\\[3pt]
\tfrac14 & \tfrac34 & 0 & 0\\[3pt]
\tfrac16 & \tfrac{5}{12} & \tfrac{5}{12} & 0\\[3pt]
0 & \tfrac{1}{6} & \tfrac{5}{12} & \tfrac{5}{12}
\end{pmatrix}.
$$

Note that in this case all nodes $u,v,w$ have paths to reach the sink $\tau$ in both heads independently, unlike in the previous example. Applying matrix multiplication repeatably to each diffusion matrix we obtain the result in Figure~\ref{fig:minimax_example}, in which we can clearly see that the minimax fidelity for the multi-head system is higher than that of each head, that is $\phi^{\mathrm{multi}}_{\min}=0.417> \phi^{(h_*)}_{\min} = \phi^{(h_1)}_{\min}=0.375$ and $\phi^{\mathrm{multi}}_{\min}=0.417>\phi^{(h_2)}_{\min}=0.250$.
\label{ex:example}
\end{example}

\begin{figure}[htbp]
  \centering
  \subfigure[Head 1]{%
    \includegraphics[width=0.32\textwidth]{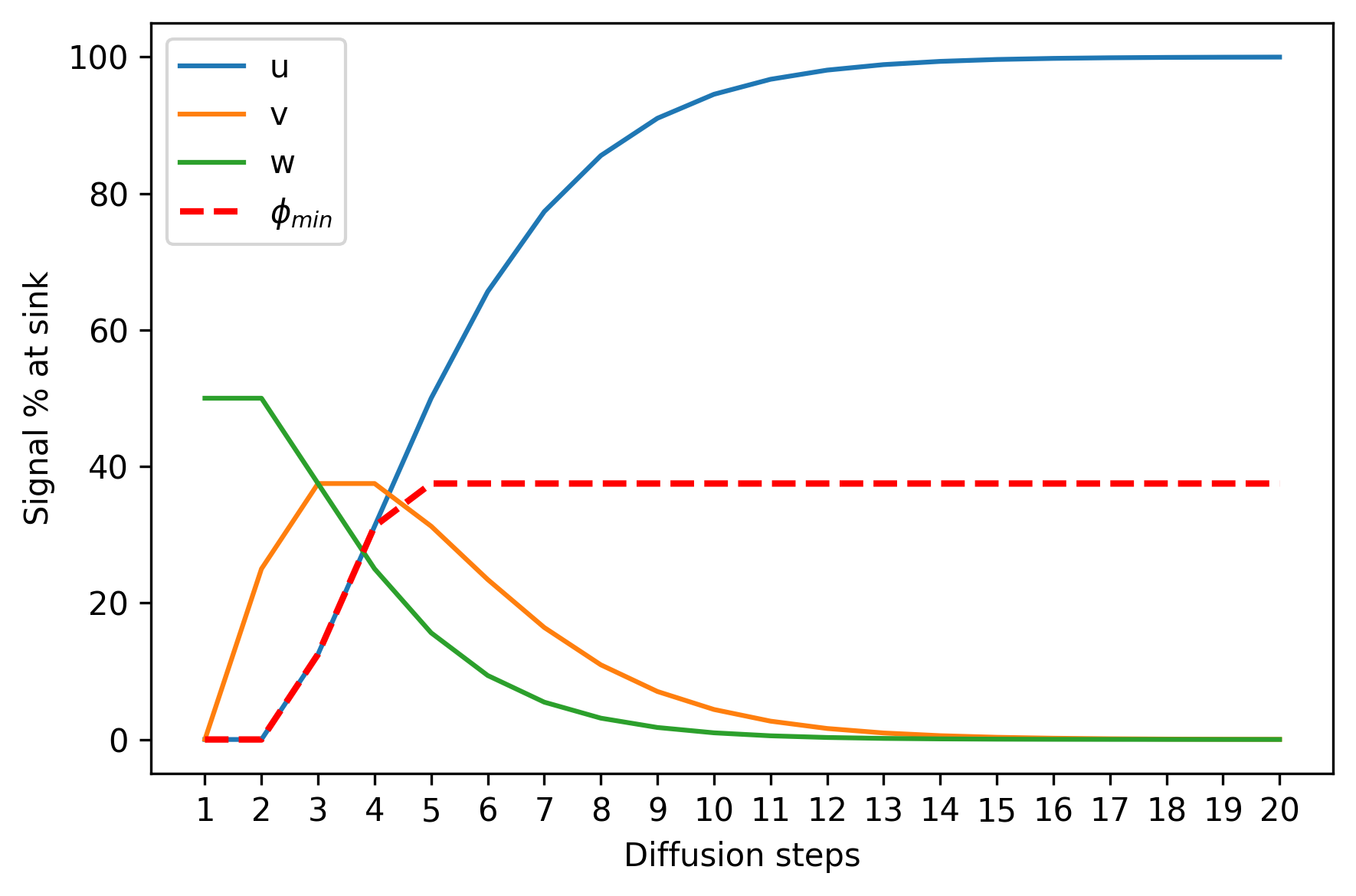}%
    \label{fig:sub1}
  }\hfill
  \subfigure[Head 2]{%
    \includegraphics[width=0.32\textwidth]{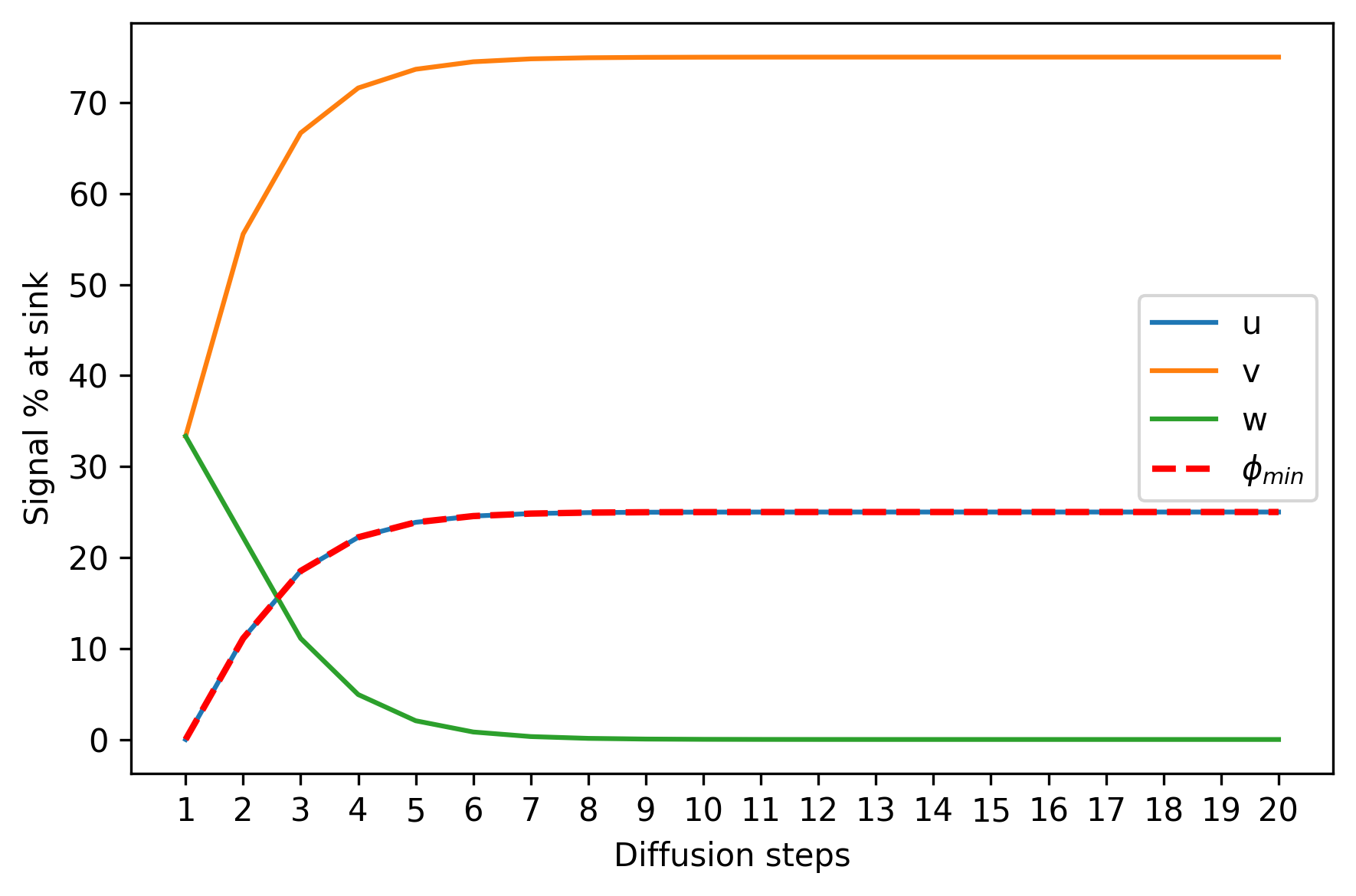}%
    \label{fig:sub2}
  }\hfill
  \subfigure[Multi-Head]{%
    \includegraphics[width=0.32\textwidth]{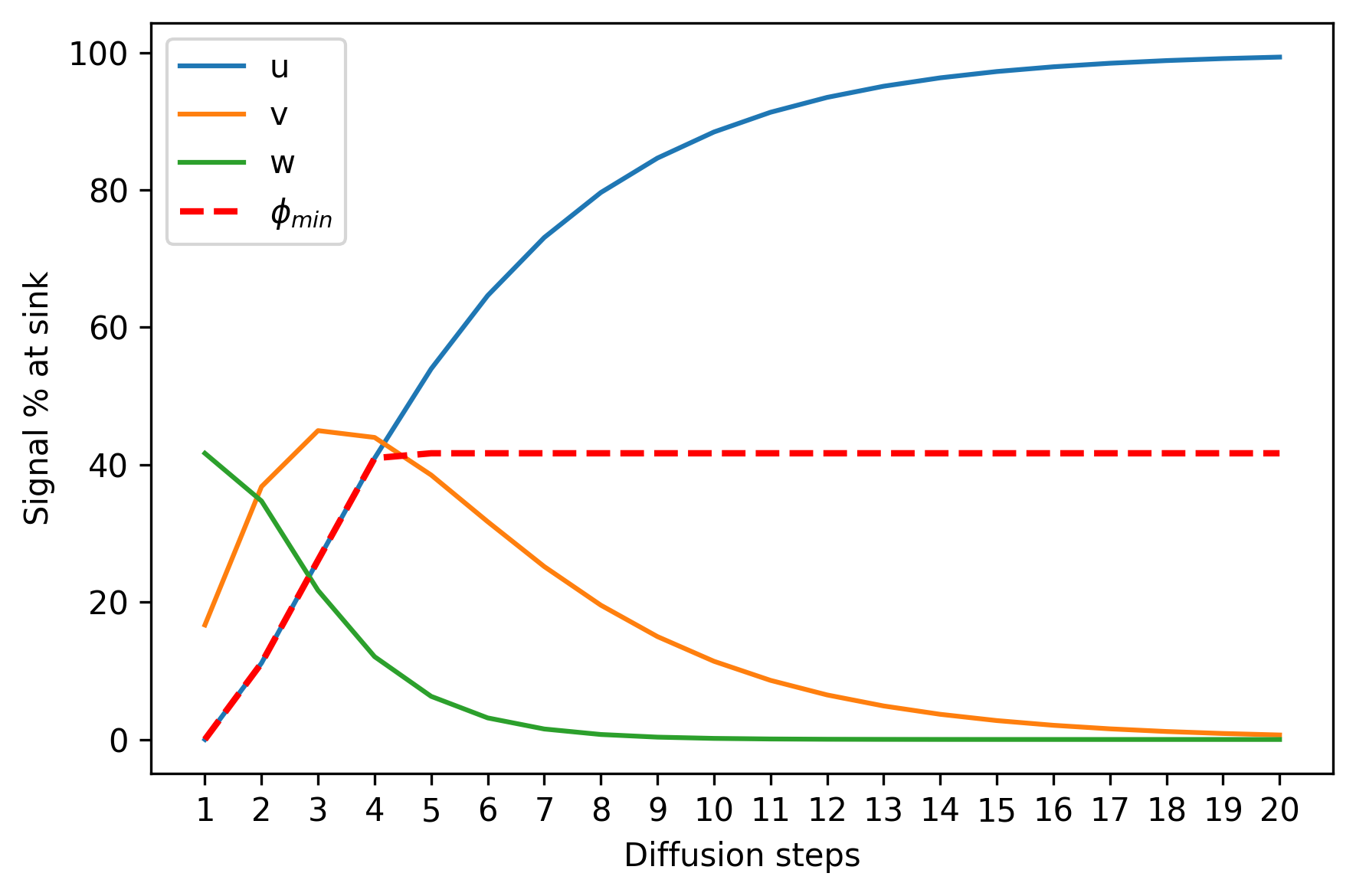}%
    \label{fig:sub3}
  }
  \caption{Diffusion of signal from nodes $u,v,w$ to the sink $\tau$ under single-head and multi-head diffusion kernels. Solid lines show signal arrival percentages over diffusion steps, while the dashed line $\phi_{\min}$ indicates the cumulative fidelity.}
  \label{fig:minimax_example}
\end{figure}

\vspace{-10pt}
This example demonstrates that combining multiple feedforward attention heads results in a multi-head system that can achieve a level of minimax fidelity exceeding that of any single head, when this is non-trivial (i.e. not zero). In a single computational graph (or single-head attention mechanism), the diffusion process is governed by one diffusion matrix, limiting the maximum achievable fidelity to that single pathway. In contrast, the multi-head setting employs diverse diffusion matrices, each capturing different aspects of signal propagation. \textit{Although one might expect the weighted average of individual fidelities to be bounded by the highest individual fidelity, the multi-head diffusion operator can surpass the best individual performance.} This synergy suggests that multi-head architectures are not merely parallel computations but can work in concert to enhance the preservation and amplification of information. For an additional discussion regarding Figure~\ref{fig:minimax_example} see Appendix~\ref{app:example}.

\section{Experimental Validation}
\label{sec:Experimental Validation}

In this section, we empirically validate our theoretical analysis by measuring proxies for mixing time and fidelity. Our experiments are conducted on two toy synthetic tasks concerned with string manipulation. We first describe the data generation and model configurations, then 
detail the algorithms used to compute each proxy, and finally present our experimental results. Tables and figures for this section can be found in Appendix~\ref{app:tablesandfigures}.
\paragraph{Data Generation} We work with two synthetic datasets. I) We generate sequences of length 100 from a vocabulary of 256 tokens. Each token is sampled independently from a discrete uniform distribution. For each sample 
$x = (x_1, x_2, \ldots, x_{100})$, the target is defined as $x$ itself: $x \mapsto x.$ A total of 5000 such examples are used for training. II) Similarly, we generate random sequences of length $n=100$. However, 
the target is produced by cyclically shifting the input by one position: $x = (x_1, x_2, \ldots, x_{99}, x_{100}) \quad \mapsto \quad (x_{100}, x_1, x_2, \ldots, x_{99}).$ This task requires the model to learn a non-trivial permutation of the sequence.

\paragraph{Model Architecture and Training} We adopt a standard pre-norm (RMSNorm) Transformer architecture with causal (decoder-only) attention. The model has 4 transformer layers with a token embedding dimension of 64 (both for the embedding and attention layers), a MultiLayer Perceptron (MLP) hidden dimension of 128, and dropout of 0.1 in both attention and MLP layers. We vary the number of attention heads over $\{1, 4, 8, 16\}$ while keeping the total embedding dimension fixed.
Thus, as the head count increases, the per-head dimension decreases accordingly. We do so to consider a constant model parameter count, and to ensure our results are solely dependent on the computational graph structure, rather than on having more trainable weights. Training is performed using the Adam optimizer with a learning rate of $10^{-3}$, a batch size of 50, and for 200 epochs.

\vspace{-10pt}
\subsection{Mixing Time Proxy Estimation}
\label{subsec:Mixing Time Proxy Estimation}

To quantify how rapidly information propagates through the model’s causal attention graph, we approximate the mixing time by a straightforward (per-layer) Monte Carlo hitting‐time experiment, see Algorithm~\ref{alg:mixing-time-per-layer} (Appendix~\ref{app:Algorithmic Descriptions}). We extract the attention tensor of shape $(H,n,n)$, where $H$ is the number of heads and $n=100$ the sequence length. We then form a single transition matrix by weighting each head according to the model’s own learned output‐projection importances. Thus, we extract each head’s contribution to the next‐layer representation by slicing the model’s output‐projection matrix into $H$ equal blocks (one per head) and computing the norm of each block.  These norms serve as raw importance scores, which we then normalize to sum to one and use to weight the corresponding $n\times n$ attention matrices.  The resulting convex combination defines a single forward‐transition matrix over token positions. Interpreting this combined matrix as defining forward-directed transition probabilities, we simulate many independent random walks starting from each token position, sampling the next position according to the attention-induced distribution over successors. Each walk proceeds until it reaches the final token or until a maximum cutoff of 100 steps is exceeded. 

Note that the attention tensor depends on the input, so we run this experiment across all samples in the dataset. We repeat this process for 500 simulations per start state and per sample to obtain an empirical mean hitting time. We would like to highlight that to be faithful to Definition~\ref{def:mixing_time}, one would need to apply the max operator across starting positions, rather than the mean as in Algorithm~\ref{alg:mixing-time-per-layer}. However, we restrict ourselves to the average hitting time across positions for computational tractability: we find that to obtain meaningful trends for the worse case mixing time it is necessary to increase the cutoff steps substantially making the simulation very expensive. In Table~\ref{tab:mixing-results-copy} and Table~\ref{tab:mixing-results-cycle} we report the results for mixing time in steps for both tasks and for all layers in the model. We also plot them in Figure~\ref{fig:mixing_copy} and Figure~\ref{fig:mixing_cycle} in which we can more clearly appreciate the trend predicted theoretically: adding more heads (under the same parameter count) leads to faster mixing.

\vspace{-10pt}
\subsection{Minimax Fidelity Proxy Estimation}
\label{subsec:Minimax Fidelity Proxy Estimation}

To quantify how sharply information from every token is preserved when it reaches the sink, we compute a (per-layer) minimax diffusion fidelity over a fixed diffusion horizon, see Algorithm~\ref{alg:minimaxfidelity-per-layer} (Appendix~\ref{app:Algorithmic Descriptions}). As before, we begin by extracting the attention tensor of shape $(H,n,n)$ and forming a single forward‐transition matrix via a convex combination of heads weighted by their learned output‐projection norms.  We then simulate the diffusion process up to 100 steps by repeatedly applying matrix multiplication: at each step we record the probability of mass starting at node $j$ arriving at the sink.  For each start node $j$ we retain its peak arrival probability over all steps, and finally we take the minimum of these peaks across $j$ to obtain the worst‐case (minimax) fidelity for that example. In Table~\ref{tab:fidelity-results-copy}, Table~\ref{tab:fidelity-results-cycle}, and Figure~\ref{fig:fidelity_copy} and Figure~\ref{fig:fidelity_cycle} we can see that the fidelity goes up with the number of heads, as predicted. Interestingly, Layer 4 of the model trained on cycle sequences appears to perform poorly in terms of both fidelity and mixing time. Similarly, when comparing subfigures in Figure~\ref{fig:performance_eval}, drops in mixing time seem to align with increases in fidelity. This could suggest a correlation between the two metrics, but further investigation is needed.

Note that thus far, we have empirically shown improved minimax fidelity for multi-head attention compared to single-head models. However, our primary discussion in Section~\ref{sec:Multi-Head Minimax Fidelity} focused on the fact that the fidelity from combining multiple heads can surpass the best individual head's fidelity at a given node within the same multi-head model. To this end, we evaluated each head's minimax fidelity independently and compared it directly with the minimax fidelity obtained by combining heads using learned weights. Empirically, we indeed observed multiple instances across both tasks (copy and cycle) and various head counts (4, 8, and 16 heads) where the combined multi-head minimax fidelity exceeded the best individual head minimax fidelity, thus providing experimental support that this effect takes place in learning-based system too, not only when handcrafted as in our original example. See Table~\ref{tab:fidelity-comparison-copy} and Table~\ref{tab:fidelity-comparison-cycle} for details.

\section{Conclusion}

In this work we introduced a graph-theoretic framework that models multi-head attention as a collection of synergistic feedforward DAGs, revealing possible benefits beyond mere computational parallelism. We show that the presence of multiple heads can reduce mixing time and amplify minimax fidelity. To verify the validity of the intuition, we perform preliminary sequence manipulation experiments with single-head versus multi-head Transformers with the same parameter count, obtaining satisfactory results. We have often argued that it may be possible to find configurations that satisfy our theoretical requirements via gradient descent optimization. This is true, but in practice the model is trying to optimize the downstream loss function: cross-entropy in the case of pre-training for next-token prediction. This implies that mixing time and minimax fidelity can only be optimized as a consequence of improving downstream performance but not directly. The model must find them useful for the task at hand. 

Finally, we note that previous work such as that by~\citep{voita2019analyzingmultiheadselfattentionspecialized,NEURIPS2019_2c601ad9} has found that many attention heads can be pruned without significant loss in performance. This does not directly contradict our findings: rather, it suggests that the benefits we identify (reduced mixing time and enhanced minimax fidelity) may plateau or saturate beyond a certain number of heads. That is, while a few heads could contribute importantly to synergy, adding more may yield diminishing returns unless head diversity is preserved. Another interesting observation is that pruning (including head pruning) is usually done after training~\citep{cheng2024surveydeepneuralnetwork}, which could suggest that additional heads help during optimization, even if they become redundant at convergence.

\bibliography{pmlr-sample}

@misc{what,
      title={What makes a good feedforward computational graph?}, 
      author={Alex Vitvitskyi and João G. M. Araújo and Marc Lackenby and Petar Veličković},
      year={2025},
      eprint={2502.06751},
      archivePrefix={arXiv},
      primaryClass={cs.LG},
      url={https://arxiv.org/abs/2502.06751}, 
}

@article{hoeffding1963probability,
  title={Probability inequalities for sums of bounded random variables},
  author={Hoeffding, Wassily},
  journal={Journal of the American Statistical Association},
  volume={58},
  number={301},
  pages={13--30},
  year={1963}
}

@inproceedings{NEURIPS2019_2c601ad9,
 author = {Michel, Paul and Levy, Omer and Neubig, Graham},
 booktitle = {Advances in Neural Information Processing Systems},
 editor = {H. Wallach and H. Larochelle and A. Beygelzimer and F. d\textquotesingle Alch\'{e}-Buc and E. Fox and R. Garnett},
 pages = {},
 publisher = {Curran Associates, Inc.},
 title = {Are Sixteen Heads Really Better than One?},
 url = {https://proceedings.neurips.cc/paper_files/paper/2019/file/2c601ad9d2ff9bc8b282670cdd54f69f-Paper.pdf},
 volume = {32},
 year = {2019}
}

@book{Levin2017MarkovCA,
  title     = {Markov Chains and Mixing Times},
  author    = {David A. Levin and Yuval Peres},
  edition   = {2},
  year      = {2017},
  publisher = {American Mathematical Society},
  address   = {Providence, RI, USA},
  isbn      = {978-1-4704-2962-1},
  series    = {American Mathematical Society Textbooks}
}

@article{Brown2020LanguageMA,
  title={Language Models are Few-Shot Learners},
  author={Tom B. Brown and Benjamin Mann and Nick Ryder and Melanie Subbiah and Jared Kaplan and Prafulla Dhariwal and Arvind Neelakantan and Pranav Shyam and Girish Sastry and Amanda Askell and Sandhini Agarwal and Ariel Herbert-Voss and Gretchen Krueger and Tom Henighan and Rewon Child and Aditya Ramesh and Daniel M. Ziegler and Jeff Wu and Clemens Winter and Christopher Hesse and Mark Chen and Eric Sigler and Mateusz Litwin and Scott Gray and Benjamin Chess and Jack Clark and Christopher Berner and Sam McCandlish and Alec Radford and Ilya Sutskever and Dario Amodei},
  journal={ArXiv},
  year={2020},
  volume={abs/2005.14165},
  url={https://api.semanticscholar.org/CorpusID:218971783}
}

@inproceedings{Vaswani2017AttentionIA,
  title={Attention is All you Need},
  author={Ashish Vaswani and Noam M. Shazeer and Niki Parmar and Jakob Uszkoreit and Llion Jones and Aidan N. Gomez and Lukasz Kaiser and Illia Polosukhin},
  booktitle={Neural Information Processing Systems},
  year={2017},
  url={https://api.semanticscholar.org/CorpusID:13756489}
}

@book{Horn1985MatrixA,
  title     = {Matrix Analysis},
  author    = {Roger A. Horn and Charles R. Johnson},
  year      = {1985},
  publisher = {Cambridge University Press},
  address   = {Cambridge, United Kingdom},
  isbn      = {978-0-521-30586-0}
}

@misc{cheng2024surveydeepneuralnetwork,
      title={A Survey on Deep Neural Network Pruning-Taxonomy, Comparison, Analysis, and Recommendations}, 
      author={Hongrong Cheng and Miao Zhang and Javen Qinfeng Shi},
      year={2024},
      eprint={2308.06767},
      archivePrefix={arXiv},
      primaryClass={cs.LG},
      url={https://arxiv.org/abs/2308.06767}, 
}

@misc{voita2019analyzingmultiheadselfattentionspecialized,
      title={Analyzing Multi-Head Self-Attention: Specialized Heads Do the Heavy Lifting, the Rest Can Be Pruned}, 
      author={Elena Voita and David Talbot and Fedor Moiseev and Rico Sennrich and Ivan Titov},
      year={2019},
      eprint={1905.09418},
      archivePrefix={arXiv},
      primaryClass={cs.CL},
      url={https://arxiv.org/abs/1905.09418}, 
}

\clearpage
\appendix

\section{Proofs}
\label{app:Proofs}

This appendix contains the proofs that complement the discussion in the main text.

\begin{proof}[Proof of Lemma \ref{lem:multi-head}]
For any head $h$, the graph $G^{(h)}$ is feedforward. This means that any random walk starting at $j$ within head $h$ can only transition to vertices $j \le i$. Furthermore, because $\tau=n$ is the unique sink for all heads, probability mass within any single head $h$ eventually flows towards $\tau$. Specifically, applying the walk matrix $W^{(h)}$ repeatedly will concentrate the mass at $\tau$. The merging operation combines the outputs from each head at each position. Crucially, this merging process is typically local and acyclic with respect to the sequence positions. For example, concatenating head outputs at position $i$ and projecting them linearly combines information already arrived at position $i$ via paths $(j \to i)$ within the various heads. This merging step does not create new paths from $\tau$ back to earlier positions $j < \tau$, nor does it allow information at a position $j < \tau$ to remain indefinitely without flowing towards $\tau$.
\end{proof}

\begin{proof}[Proof of Theorem \ref{thm:multi-head-mixing}]
For each step of the combined Markov chain defined by \(\overline{W}\), the probability of making a forward move is at least \(p\). Let \(X\) denote the total number of forward moves made after \(t\) independent steps. Then
\begin{equation}
  X \sim \operatorname{Bin}(t,p)
\end{equation}
and the expected number of forward moves is
\begin{equation}
  \mathbb{E}[X] = t\,p.
\end{equation}
Setting the expected number of forward moves to be \(2N\) (i.e., \(t\,p = 2N\)) gives
\begin{equation}
  t = \frac{2N}{p}.
\end{equation}
Now, by Hoeffding's inequality~\citep{hoeffding1963probability},
\begin{equation}
  \Pr\bigl[X < N\bigr] \le \exp\!\Bigl(-\frac{2(tp-N)^2}{t}\Bigr).
\end{equation}
Substituting \(tp = 2N\) yields
\begin{equation}
  \Pr\bigl[X < N\bigr] \le \exp\!\Bigl(-\frac{2N^2}{2N/p}\Bigr)
  = \exp(-pN).
\end{equation}
For fixed \(p\) and sufficiently large \(N\), the right-hand side becomes very small (e.g., less than \(1/4\)), meaning that with high probability the chain makes at least \(N\) forward moves within \(t = \frac{2N}{p}\) steps. In other words, the mixing time is bounded by
\begin{equation}
  T_{\mathrm{mix}}(\overline{W},\epsilon) \lesssim \frac{2N}{p}.
\end{equation}

The above refers to the mixing time defined with respect to total variation distance, i.e., the time required for the Markov chain to be within \(\epsilon\) of the stationary distribution. For the combined chain this quantity is no more than a constant multiple of $\frac{2N}{p}$ (or behaves asymptotically like $\frac{2N}{p}$) as the relevant parameters vary.

Next, for \(H\) heads with individual forward probabilities \(p_1,\dots,p_H\) and convex weights \(\alpha_1,\dots,\alpha_H\) (with \(\alpha_h\ge 0\) and \(\sum_{h=1}^{H} \alpha_h = 1\)), the effective forward probability is  
\begin{equation}
p = \sum_{h=1}^{H} \alpha_h\,p_h.
\end{equation}
Since the weights are nonnegative and sum to 1, by the properties of convex combinations we have  
\begin{equation}
\min_{1\le h\le H} p_h \le p \le \max_{1\le h\le H} p_h.
\end{equation}
Hence, if one could choose the weights optimally (let us assume that gradient descent will likely try to optimize for an optimal combination)—for example, by setting \(\alpha_{h^*}=1\) for the head \(h^*\) that attains the maximum \(p_h\) and \(\alpha_h=0\) for all \(h\ne h^*\)—then we would have  
\begin{equation}
p^* = \max_{1\le h\le H} p_h.
\end{equation}
In that ideal case, the overall mixing time (which is roughly bounded by \(\frac{2N}{p}\)) would be  
\begin{equation}
T_{\mathrm{mix}}(\overline{W},\epsilon) \lesssim \frac{2N}{\max_{1\le h\le H} p_h},
\end{equation}
which is as fast as the fastest individual head. The result simply emphasizes that one cannot hope for a combined chain to mix faster than the fastest individual chain, though with optimal (or adaptive) weighting one can match it.
 
Even if one uses fixed (say, uniform) weights instead of the idealized choice, the presence of multiple heads means that with high probability at least one head has a relatively high \(p_h\). To show this formally assume that the forward probabilities \(p_1, p_2, \dots, p_H\) for the \(H\) heads are independent random variables with a common cumulative distribution function~(CDF), \(F(p)\). For any fixed threshold \(\bar{p}\) such that \(F(\bar{p}) < 1\), the probability that a single head satisfies \(p_h \ge \bar{p}\) is
\begin{equation}
\Pr(p_h \ge \bar{p}) = 1 - F(\bar{p}).
\end{equation}
Since the heads are independent, the probability that all \(H\) heads have \(p_h < \bar{p}\) is
\begin{equation}
\Pr\Bigl(p^* < \bar{p}\Bigr) = \Bigl(F(\bar{p})\Bigr)^H.
\end{equation}
Thus, the probability that at least one head achieves \(p_h \ge \bar{p}\) is
\begin{equation}
\Pr\Bigl(p^* \ge \bar{p}\Bigr) = 1 - \Bigl(F(\bar{p})\Bigr)^H.
\end{equation}
Since \(F(\bar{p}) < 1\), as \(H\) increases,
\begin{equation}
\lim_{H\to\infty} \Pr\Bigl(p^* \ge \bar{p}\Bigr) = 1.
\end{equation} 

Consequently, if one could choose weights adaptively to favor the head with the highest $p_h$, for instance using gradient descent, \(p\) will be statistically biased toward higher values. While the worst-case bound is  
\begin{equation}
\frac{2N}{\max_{h} p_h} \le T_{\mathrm{mix}}(\overline{W},\epsilon) \lesssim \frac{2N}{\min_{h} p_h},
\end{equation}
the key point is that, statistically, multiple heads increase the chance that a high \(p_h\) is present, thereby allowing the effective mixing time to approach the lower bound of \(\frac{2N}{\max_{h} p_h}\).
\end{proof}

\begin{remark}[Scope and Modeling Assumptions in Theorem~\ref{thm:multi-head-mixing}]
In real multi-head attention, heads are concatenated and passed through a linear projection \(W_{\mathrm{proj}}\), which produces a new row-stochastic, input-dependent kernel. This is \emph{not} literally a convex combination of the \(W^{(h)}\). Our model uses such a combination as a \emph{tractable analytical proxy} to upper-bound the effective forward propagation strength, assuming \(W_{\mathrm{proj}}\) preserves the feedforward DAG structure. Additionally, we would like to clarify that our mixing time analysis models information propagation \emph{within a single application} of multi-head attention (i.e., one forward pass through an attention layer). For a fixed input, each head defines a fixed transition matrix \(W^{(h)}\), and the merged kernel \(\overline{W} = \sum_{h} \alpha_h W^{(h)}\) is also fixed. Thus, the Markov chain defined by iterating \(\overline{W}^t\) is \emph{time-homogeneous}, and our bound applies directly. Note that \emph{time-inhomogeneous} chains arise naturally in \emph{multi-layer} Transformers, where each layer applies a distinct attention kernel. Our analysis is restricted to \emph{intra-layer} dynamics and does \emph{not} extend across layers.
\end{remark}

\section{Example~\ref{ex:example} temporal dynamics discussion}
\label{app:example}

Figure~\ref{fig:minimax_example} illustrates the temporal dynamics of Example~\ref{ex:example}. Each head defines a distinct diffusion pattern: the first head exhibits a gradual but comprehensive propagation from $u$ to $\tau$, while the second head emphasizes faster, more localized diffusion primarily through $v$. When combined into a multi-head system, the resulting diffusion operator integrates these complementary pathways, producing a joint evolution that exceeds the minimax fidelity achievable by either head alone. In the combined diffusion, the node signals display cooperative behavior: $u$ maintains long-range reachability while $v$ contributes strong early propagation.

\section{Compute Resources}

All experiments were performed on a single H100 GPU under a Linux environment.

\section{Results: Tables and Figures}
\label{app:tablesandfigures}

In this appendix we include the tables and figures discussed in Section~\ref{sec:Experimental Validation}.

\begin{table}[hbpt!]
\centering
\caption{Mixing time steps mean $\pm$ std dev. for \textbf{Copy Sequence}.}
\label{tab:mixing-results-copy}
\scalebox{0.5}{
\begin{tabular}{lcccc}
\toprule
\textbf{Heads} & L1 & L2 & L3 & L4 \\
\midrule
1  & $3.8344 \pm 0.1391$ & $77.9281 \pm 2.8387$ & $78.6635 \pm 3.2322$ & $73.5020 \pm 3.6750$ \\
4  & $3.7370 \pm 0.0259$ & $3.7155 \pm 0.0365$ & $3.8895 \pm 0.0704$ & $7.8050 \pm 0.5391$ \\
8  & $3.6793 \pm 0.0171$ & $3.6821 \pm 0.0329$ & $3.7952 \pm 0.0661$ & $3.7467 \pm 0.0731$ \\
16 & $3.6328 \pm 0.0150$ & $3.6513 \pm 0.0212$ & $3.6391 \pm 0.0564$ & $3.6594 \pm 0.0687$ \\
\bottomrule
\end{tabular}
}
\end{table}
\begin{table}[hbpt!]
\centering
\caption{Mixing time steps mean $\pm$ std dev. for \textbf{Cycle Sequence}.}
\label{tab:mixing-results-cycle}
\scalebox{0.5}{
\begin{tabular}{lcccc}
\toprule
\textbf{Heads} & L1 & L2 & L3 & L4 \\
\midrule
1  & $27.0384 \pm 2.4355$ & $47.8151 \pm 0.6754$ & $47.2291 \pm 1.3803$ & $43.0723 \pm 2.8604$ \\
4  & $3.6369 \pm 0.0161$ & $3.9092 \pm 0.0374$ & $45.7709 \pm 1.9865$ & $31.2925 \pm 3.3583$ \\
8  & $3.6426 \pm 0.0098$ & $3.7756 \pm 0.0290$ & $4.6735 \pm 0.0795$ & $47.5304 \pm 0.9180$ \\
16 & $3.5859 \pm 0.0118$ & $3.7820 \pm 0.0307$ & $4.3759 \pm 0.0583$ & $41.9292 \pm 1.9165$ \\
\bottomrule
\end{tabular}
}
\end{table}

\begin{table}[htbp!]
\centering
\caption{Fidelity $\%$ mean $\pm$ std dev. for \textbf{Copy Sequence}}
\label{tab:fidelity-results-copy}
\scalebox{0.5}{
\begin{tabular}{lcccc}
\toprule
\textbf{Heads} & L1 & L2 & L3 & L4 \\
\midrule
1  & $0.3000 \pm 0.1000$ & $0.0000 \pm 0.0000$ & $0.0000 \pm 0.0000$ & $0.0000 \pm 0.0000$ \\
4  & $0.4500 \pm 0.0700$ & $0.2800 \pm 0.0800$ & $0.1800 \pm 0.0500$ & $0.0100 \pm 0.0100$ \\
8  & $0.5000 \pm 0.0800$ & $0.3500 \pm 0.0800$ & $0.1900 \pm 0.0800$ & $0.1600 \pm 0.0600$ \\
16 & $0.5800 \pm 0.0600$ & $0.4900 \pm 0.0700$ & $0.3400 \pm 0.0700$ & $0.2500 \pm 0.0700$ \\
\bottomrule
\end{tabular}
}
\end{table}

\begin{table}[htbp!]
\centering
\caption{Fidelity $\%$ mean $\pm$ std dev. for \textbf{Cycle Sequence}}
\label{tab:fidelity-results-cycle}
\scalebox{0.5}{
\begin{tabular}{lcccc}
\toprule
\textbf{Heads} & L1 & L2 & L3 & L4 \\
\midrule
1  & $0.0000 \pm 0.0000$ & $0.0000 \pm 0.0000$ & $0.0000 \pm 0.0000$ & $0.0000 \pm 0.0000$ \\
4  & $0.5300 \pm 0.0800$ & $0.4800 \pm 0.0600$ & $0.0000 \pm 0.0000$ & $0.0000 \pm 0.0000$ \\
8  & $0.5400 \pm 0.0700$ & $0.6300 \pm 0.0500$ & $0.5300 \pm 0.0700$ & $0.0000 \pm 0.0000$ \\
16 & $0.6500 \pm 0.0600$ & $0.6000 \pm 0.0500$ & $0.5200 \pm 0.0500$ & $0.0000 \pm 0.0000$ \\
\bottomrule
\end{tabular}
}
\end{table}

\begin{figure}[htbp]
    \centering
    \subfigure[Mixing time for Sequence Copying]{%
        \includegraphics[width=0.38\textwidth]{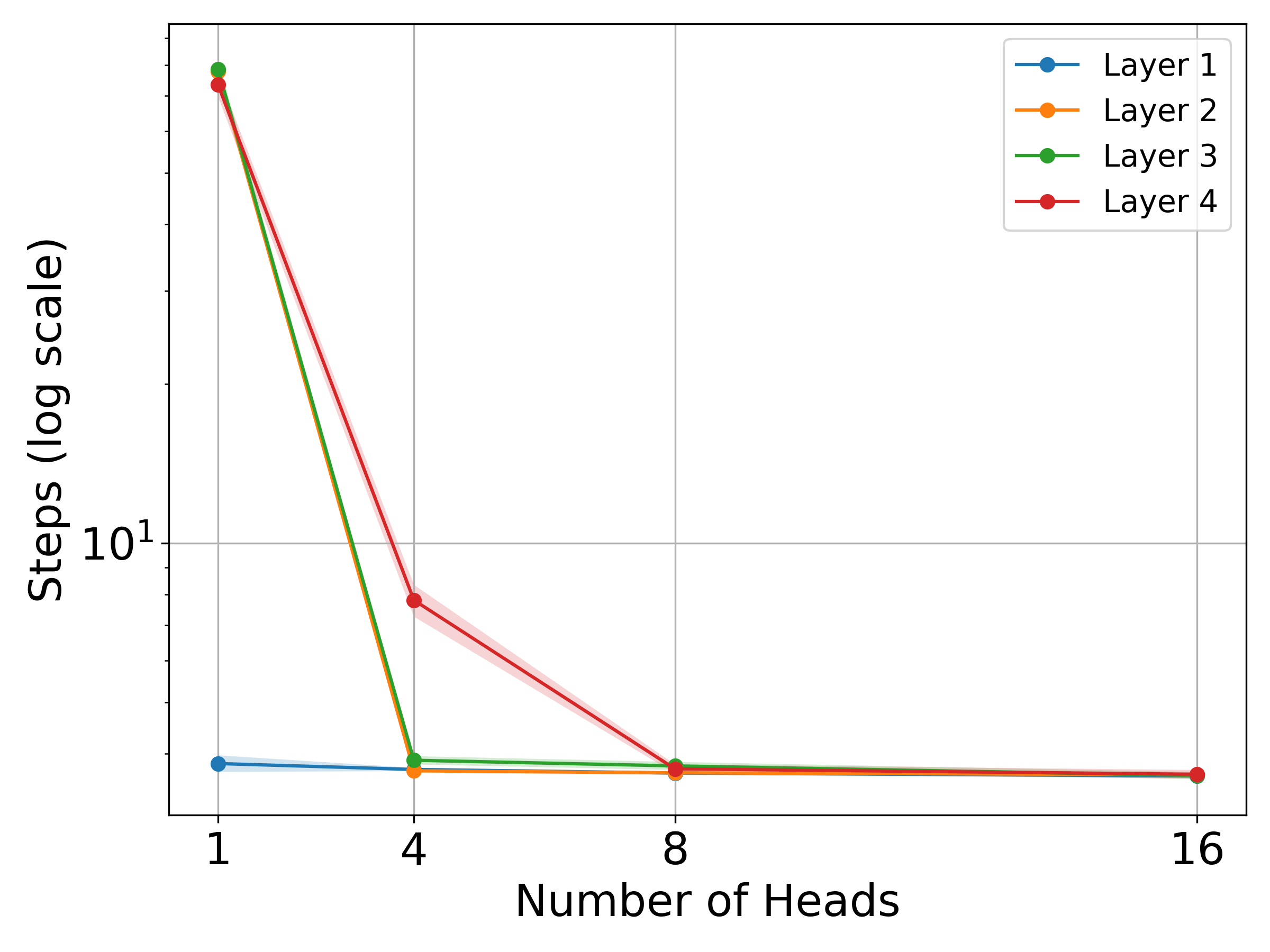}%
        \label{fig:mixing_copy}
    }\hfill
    \subfigure[Mixing time for Sequence Cycling]{%
        \includegraphics[width=0.38\textwidth]{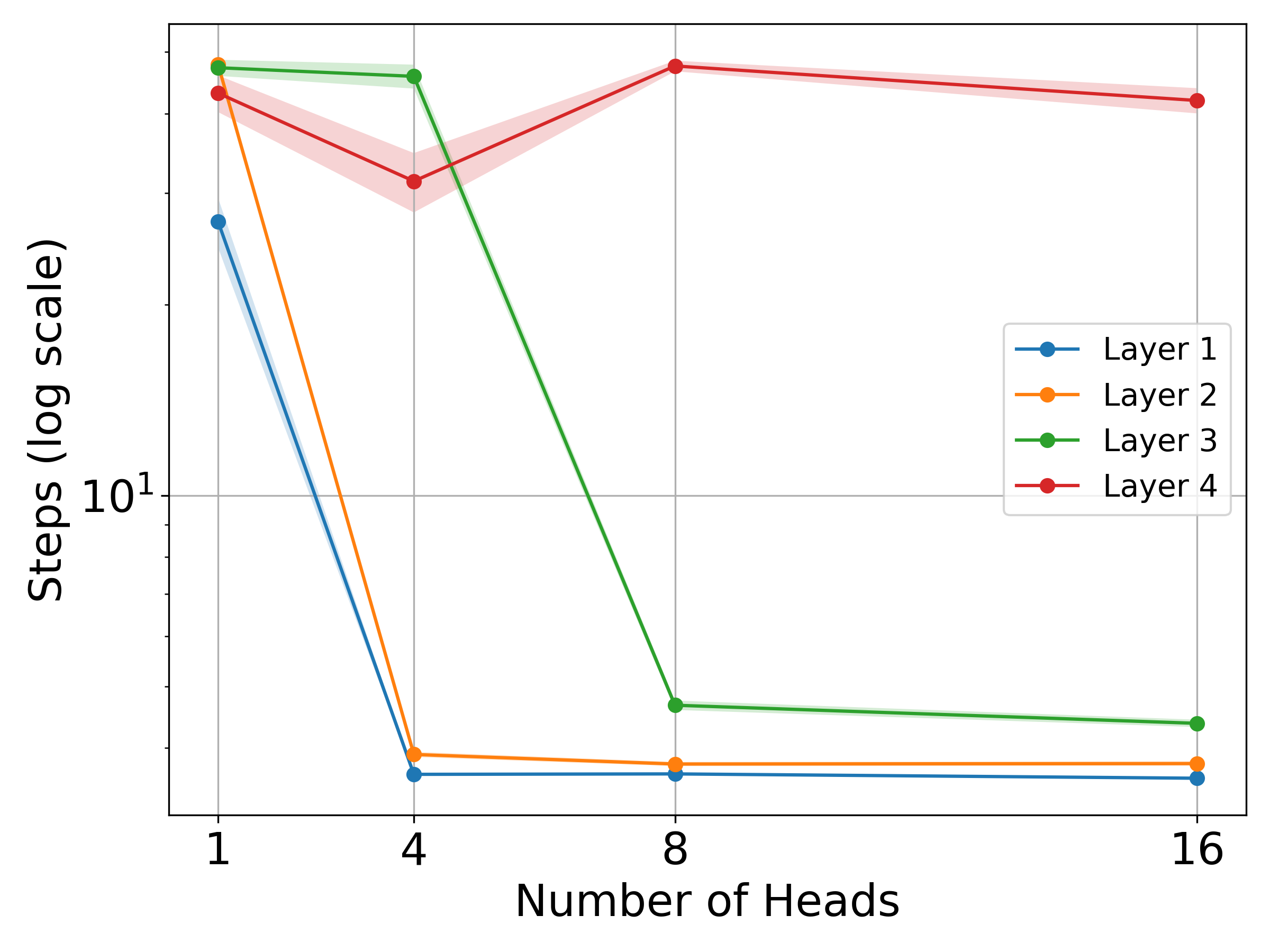}%
        \label{fig:mixing_cycle}
    }\\[2ex]
    \subfigure[Fidelity for Sequence Copying]{%
        \includegraphics[width=0.38\textwidth]{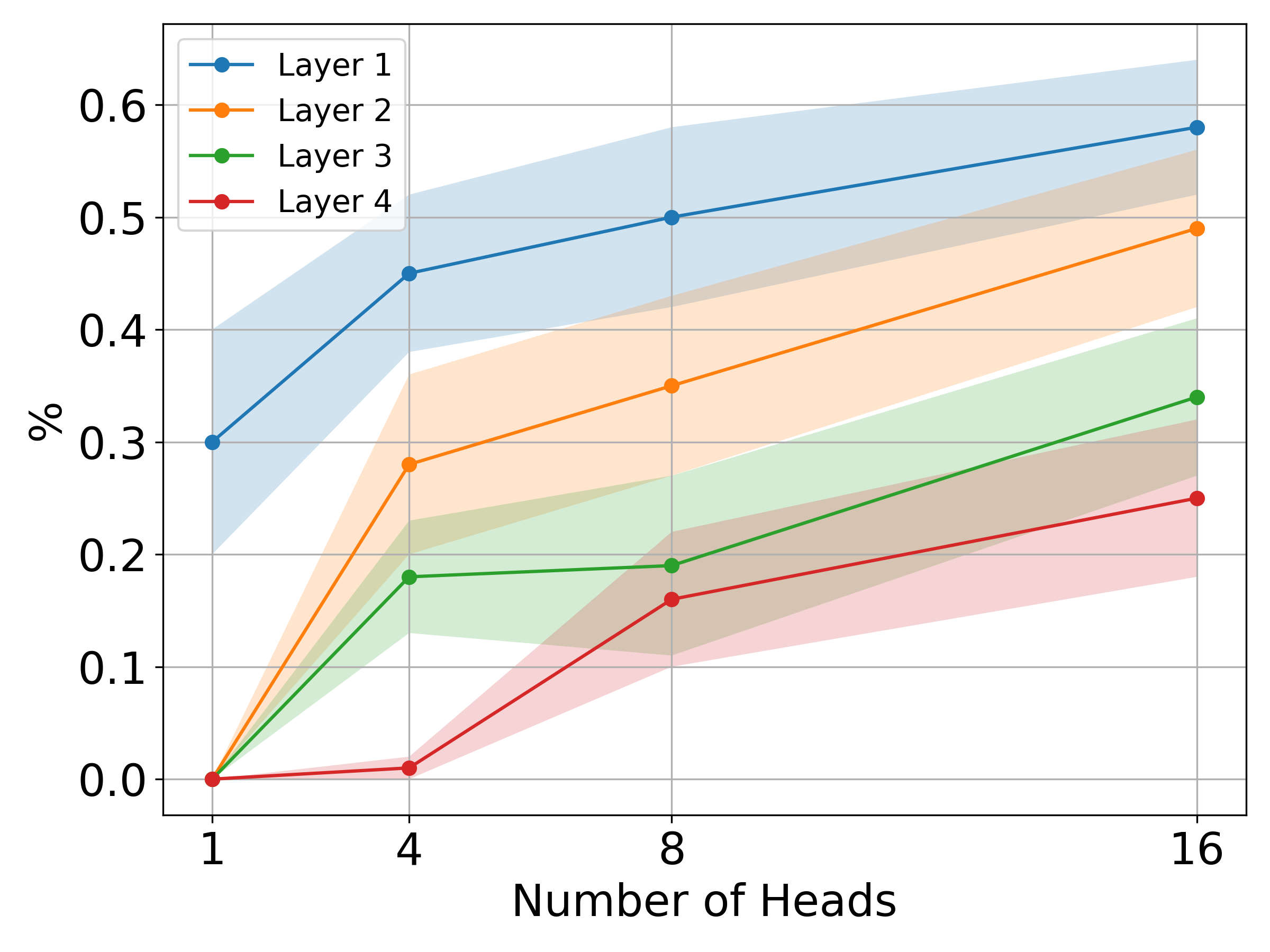}%
        \label{fig:fidelity_copy}
    }\hfill
    \subfigure[Fidelity for Sequence Cycling]{%
        \includegraphics[width=0.38\textwidth]{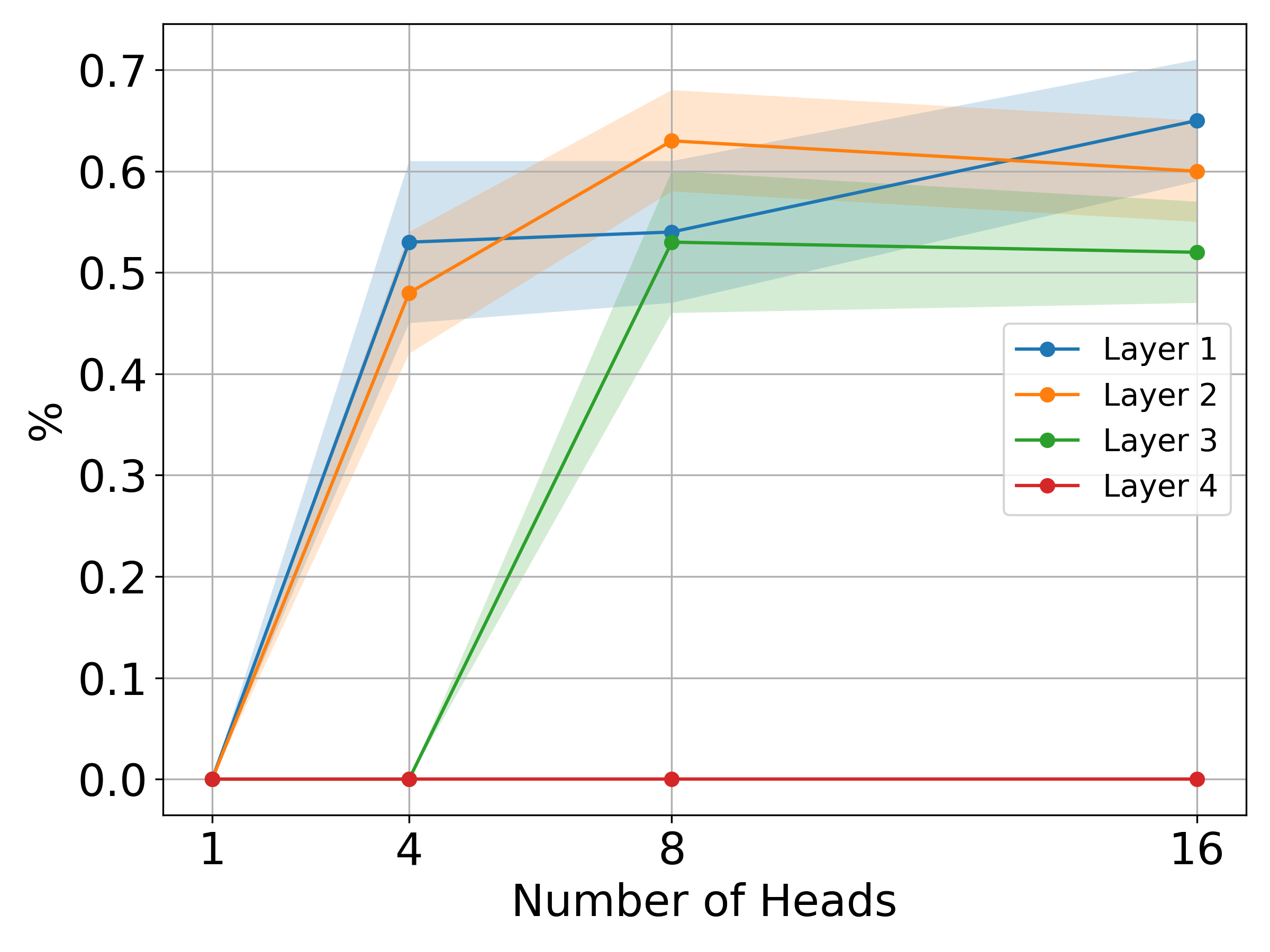}%
        \label{fig:fidelity_cycle}
    }
    \caption{Mixing time and fidelity for Transformer model trained on synthetic sequence manipulation tasks.}
    \label{fig:performance_eval}
\end{figure}

\begin{table}[htbp!] \centering \caption{Best individual vs combined multi-head fidelity $\%$ for \textbf{Copy Sequence}.} \label{tab:fidelity-comparison-copy} \scalebox{0.44}{ \begin{tabular}{lcccc} \toprule Heads & L1 & L2 & L3 & L4 \\ \midrule 4 (Individual) & $0.30$ & $0.11$ & $0.21$ & $0.00$ \\ 4 (Combined) & $\textbf{0.56}$ & $\textbf{0.43}$ & $0.19$ & $0.00$ \\ \midrule 8 (Individual) & $\textbf{0.59}$ & $0.29$ & $\textbf{0.18}$ & $\textbf{0.18}$ \\ 8 (Combined) & $0.51$ & $\textbf{0.37}$ & $0.10$ & $0.17$ \\ \midrule 16 (Individual) & $\textbf{0.59}$ & $\textbf{0.59}$ & $\textbf{0.52}$ & $0.22$ \\ 16 (Combined) & $\textbf{0.59}$ & $0.51$ & $0.46$ & $\textbf{0.23}$ \\ \bottomrule \end{tabular} } \end{table}

\begin{table}[htbp!] \centering \caption{Best individual vs combined multi-head fidelity $\%$ for \textbf{Cycle Sequence}.} \label{tab:fidelity-comparison-cycle} \scalebox{0.44}{ \begin{tabular}{lcccc} \toprule Heads & L1 & L2 & L3 & L4 \\ \midrule 4 (Individual) & $0.39$ & $0.39$ & $0.00$ & $0.00$ \\ 4 (Combined) & $\textbf{0.59}$ & $\textbf{0.50}$ & $0.00$ & $0.00$ \\ \midrule 8 (Individual) & $0.35$ & $0.49$ & $0.51$ & $0.00$ \\ 8 (Combined) & $\textbf{0.59}$ & $\textbf{0.72}$ & $\textbf{0.55}$ & $0.00$ \\ \midrule 16 (Individual) & $0.65$ & $0.49$ & $\textbf{0.68}$ & $0.00$ \\ 16 (Combined) & $\textbf{0.67}$ & $\textbf{0.55}$ & $0.48$ & $0.00$ \\ \bottomrule \end{tabular} } \end{table}

\clearpage

\section{Algorithmic Descriptions}
\label{app:Algorithmic Descriptions}

In this appendix, we include the algorithmic descriptions discussed in Section~\ref{sec:Experimental Validation}.

\begin{algorithm}[hbtp!]
\caption{Monte Carlo Mixing‐Time Proxy}
\label{alg:mixing-time-per-layer}
\SetKwInOut{Require}{Require}
\SetKwInOut{Compute}{Compute}
\Require{Dataset $D$, with samples indexed by $b \in \{1,\dots,|D|\}$, attention $\{A_b^{(h)}\}_{h=1}^H$, head weights $\{w_b^{(h)}\}_{h=1}^H$, sequence length $n$, number of simulations $S$, max steps $M$}
\Compute{Mixing times $T_{mix}$}
\BlankLine

  \For{$b \leftarrow 1$ \KwTo $|D|$}{
  Compute transition matrix $P_b
      \leftarrow \sum_{h=1}^H w_b^{(h)} \,\bigl(A_b^{(h)}\bigr)^\top$
  \For{$i \leftarrow 1$ \KwTo $n$}{
    \For{$s \leftarrow 1$ \KwTo $S$}{
      $T_{b,i,s} \leftarrow 0$ \quad $j \leftarrow i,\quad t \leftarrow 0$ \\
      \While{$j \neq n$ {and} $t < M$}{
        sample $j \sim P_b(j)$ \\
        $t \leftarrow t + 1$\;
      }
      $T_{b,i,s} \leftarrow T_{b,i,s} + t$\;
    }}}
    $T_{mix} \leftarrow \frac{1}{|D|} \frac1S \frac1n   \sum_{b=1}^{b=|D|}    \sum_{s=1}^{s=S} \sum_{i=1}^{i=n}  T_{b,i,s}$\;
  
\Return $T_{mix}$
\end{algorithm}

\begin{algorithm}[hbtp!]
\caption{Minimax Fidelity Proxy}
\label{alg:minimaxfidelity-per-layer}
\SetKwInOut{Require}{Require}
\SetKwInOut{Compute}{Compute}
\Require{
    Dataset $D$, with samples indexed by $b \in \{1,\dots,|D|\}$, attention matrices $\{A_b^{(h)}\}_{h=1}^{H}$ for each sample, head weights $\{w_b^{(h)}\}_{h=1}^{H}$, diffusion horizon $M$, target (sink) position $\tau$.
}
\Compute{Minimax Fidelity $\phi^{\mathrm{multi}}_{\min}$}
\BlankLine
    \For{$b \leftarrow 1$ \KwTo $|D|$} {
  $P_b \leftarrow \sum_{h=1}^H w_b^{(h)} \,\bigl(A_b^{(h)}\bigr)^\top$ \;
  $F_{b,m=0,ij} \leftarrow I$ \;
  \For{$m \leftarrow 1$ \KwTo $M$} {
    $F_{b,m,ij} \leftarrow \text{matmul}(F_{b,m-1,ij},P_b)$ \;
  }
}
$\phi^{\mathrm{multi}}_{\min} \leftarrow |D|^{-1}\sum_{b=1}^{|D|} \text{min}_j \, \text{max}_m \, F_{b,m,i=\tau j}$ \;

\Return $\phi^{\mathrm{multi}}_{\min}$
\end{algorithm}

\end{document}